\documentclass{article}

\PassOptionsToPackage{numbers, compress, sort}{natbib}

 \usepackage[preprint]{neurips_2026}
 \usepackage{macros}

\usepackage[utf8]{inputenc} 
\usepackage[T1]{fontenc}    
\usepackage{hyperref}       
\usepackage{url}            
\usepackage{booktabs}       
\usepackage{amsfonts}       
\usepackage{nicefrac}       
\usepackage{microtype}      
\usepackage{xcolor}         

\title{Smoothed Elicitation Complexity for Approximate $\Gamma$-calibration of Discrete Classification Tasks}

\author{%
  Jessica Finocchiaro \\
  Computer Science, Boston College \\
  \texttt{finocch@bc.edu} \\
  \And
  Victor Ganson \\
  Computer Science, Boston College \\
  \texttt{gansonv@bc.edu} \\
  \And
  Drona Khurana \\
  Computer Science, University of Colorado Boulder \\
  \texttt{drona.khurana@colorado.edu}
}

\begin{document}

\maketitle

\begin{abstract}
One prominent method of evaluating machine learning model trustworthiness is the notion of \emph{calibration}.
In the binary outcome setting, a probabilistic predictor is calibrated if outcomes are realized according to a model's distributional prediction, conditioned on this prediction.
Straightforward extensions of binary calibration definitions to probabilistic multiclass classifiers suffer from an exponential complexity blowup as the space of predictions grows exponentially in the number of classes $n$.
As a remedy, \citet{noarov_statistical_2023} propose multiclass calibration with predictions that are \emph{properties} of the outcome distribution, reducing complexity from growing in the number of classes $n$ to the \emph{dimension} $d$ of the property, called its elicitation complexity. 
Previous work on approximate property calibration is generally limited to continuous scalar properties, despite many relevant properties of interest being discrete, like the mode or rankings.
We characterize the approximate property calibration of discrete properties which are strongly orderable by using Lipschitz continuous properties as an intermediary.
This work is the first to our knowledge to provide approximate calibration results for discrete properties.
Along the way, we characterize the Lipschitz elicitation complexity of strongly orderable discrete properties by constructing algorithms for designing these Lipschitz properties, which we prove can be post-processed to obtain the original discrete property.
\end{abstract}

\section{Introduction}
One prominent paradigm of evaluating machine learning models examines \emph{calibration}, where a model's probabilistic predictions are roughly realized. 
That is, instead of conditioning on a feature instance $X = x$ (as desired in evaluating Bayes predictors), we condition on a model's prediction $f(x) = p$, and want to observe outcomes roughly aligned with $p$.
Intuitively, this is typically formalized as $\E[Y|f(X) = p] \approx p$ for all $p \in \im(f)$.
Empirically, continuous predictors must have their predictions \emph{binned} to obtain multiple instances of the same conditioning event.
The sample complexity of calibration is known to grow polynomially in the number of bins on the space~\citep{zhao_calibrating_2021,gopalan_computationally_2024,collina2026samplecomplexitymulticalibration}.

This poses a prohibitive challenge for \emph{multiclass} calibration: when bins represent approximate probabilistic predictions, the number of bins na\"ively grows exponentially in the number of classes, posing substantial computational and statistical challenges.
To circumvent this, many works on multiclass calibration reduce the granularity of the binned space~\citep{zhao_calibrating_2021,bairaktari2025sample,kull_beyond_2019,gopalan_computationally_2024,gopalan2022low,hu2025efficientswapmulticalibrationelicitable}.
These approaches generally combine full distributional predictions with some polynomial binning approach, and often lack decision-theoretic guarantees for the task implicitly guiding the polynomial binning.
\citet{derr2025threetypesofcalibration} outline a few possible approaches to defining multiclass calibration, consolidating many of these approaches to what they call \emph{distribution calibration with respect to $\gamma$}, where $\gamma : \simplex \to \R$ maps distributions to reports is called a \emph{property} guiding the binning approach.
In contrast, they show distribution calibration with respect to $\gamma$ is a stronger notion than an alternative called $\Gamma$-calibration--- initially introduced by \citet{noarov_statistical_2023,gneiting_regression_2023}, in which a model estimates the property $\gamma$ itself--- rather than a distribution--- while still providing decision-theoretic guarantees.

Critically, results surrounding approximate $\gamma$-calibration assume the prediction target $\gamma$ is \emph{continuous}.
These \emph{continuous} (calibrated) predictions are often used to make \emph{discrete} decisions.
This continuous-to-discrete jump often poses challenges for obtaining decision-theoretic guarantees, as demonstrated by \citet{noarov_statistical_2023}'s assumption of continuous properties and \citet{derr2025threetypesofcalibration} only giving exact calibration results for discrete properties.
From a technical perspective, it is difficult to obtain \emph{approximate} calibration guarantees for discrete prediction tasks since discrete predictions do not provide the opportunity to communicate a model's uncertainty, which is critical to understand if a model is ``close to'' calibration.
More precisely, this tension emerges from the need for a calibration error metric that is \emph{continuous} in the prediction space, which is impossible if predictions themselves are not continuous.

To bridge this gap, we propose using \emph{smoothed property elicitation} to obtain a continuous predictor of an intermediate continuous property $\Gamma$ which can be post-processed to obtain the desired discrete decision property $\gamma$. 
We can obtain approximate calibration guarantees with this smoothed property, while bounding the probability of miscalibration in the discrete space.
The \emph{dimension} of this smoothed property is an important object of study in its own right.
If the prediction dimension $d$ is significantly lower than predicting a full distribution over $n$ labels, this improves the computational complexity of gradient-based optimization algorithms.
Moreover, the number of bins evaluated in empirical estimation of miscalibration grows exponentially in the prediction dimension $d$.
Since sample complexity grows in the number of bins, a prediction dimension $d < n$ can significantly improve computational and statistical complexity. 

\Cref{fig:flowchart} outlines the organization of the paper.
In \Cref{sec:elic-char}, we first characterize the discrete properties that can be obtained by post-processing a continuous property $\Gamma : \simplex \to \reals$ in $1$ dimension.
Our characterization is constructive, providing two algorithms yielding smoothed properties. 
With this characterization, we proceed to obtain bounds on the approximate $\Gamma$-calibration of probabilistic predictors in \Cref{subsec:dist-Gamma}, then bound the approximate discrete $\gamma$-calibration in \Cref{subsec:Gamma-gamma} by using the smoothed property as an intermediary.

\begin{figure}
    \centering
\resizebox{\linewidth}{!}{\definecolor{purple}{RGB}{209, 120, 250}

\begin{tikzpicture}[
    node distance=1.5cm and 1cm,
    start chain=going right,
    block/.style={
        rectangle, 
        draw=black, 
        thick, 
        fill=purple!5,
        text centered, 
        rounded corners, 
        minimum height=3.5em, 
        text width=3.5cm,
        on chain
    },
    arrow/.style={
        -{Stealth[scale=1.2]},
        thick
    }
]

    \node [block] (discrete) {
        Discrete property $\gamma$
    };

    \node [block] (continuous) {
        Continuous property $\Gamma$ \\ \scriptsize (refining $\gamma$)
    };

    \node [block] (gamma_cal) {
        Approximate $\Gamma$-calibration of $f$
    };

    \node [block, fill=purple!90!black] (psi_cal) {
        Approximate $\gamma$-calibration of \\ $\psi \circ g \equiv \psi \circ \Gamma \circ f$
    };

    \node [block, below=0.8cm of continuous] (f) {
        Predictor $f : \X \to \simplex$ \\ \scriptsize ($\epsilon$-distribution calibrated w.r.t. $\Gamma$)
    };
    \node [block, below=0.8cm of gamma_cal] (g) {
        Predictor $g : \X \to \reals$ \\ \scriptsize ($\epsilon$-$\Gamma$ calibrated)
    };

    \draw [arrow] (discrete) -- (continuous);
    \draw [arrow] (continuous) -- (gamma_cal);
    \draw [arrow] (gamma_cal) -- (psi_cal);
    \draw [arrow] (f) -- (gamma_cal);
    \draw [arrow] (g) -- (psi_cal);

    \node[above=0.2cm of discrete, font=\itshape\small] {Input};
    \node[above=0.2cm of continuous, font=\itshape\small] {\Cref{sec:elic-char}};
    \node[above=0.2cm of gamma_cal, font=\itshape\small] {\Cref{subsec:dist-Gamma}};
    \node[above=0.2cm of psi_cal, font=\itshape\small] {\Cref{subsec:Gamma-gamma}};

\end{tikzpicture}}
\caption{Flowchart of results. While we start with a discrete property $\gamma$, we use continuous properties $\Gamma$ as intermediaries for approximate calibration. We discuss the construction of these properties in \Cref{sec:elic-char}, implications for approximate $\Gamma$-calibration in \Cref{subsec:dist-Gamma}, and further implications for approximate $\gamma$-calibration in \Cref{subsec:Gamma-gamma}, where post-processing is formalized through $\psi$.}
\label{fig:flowchart}
\end{figure}
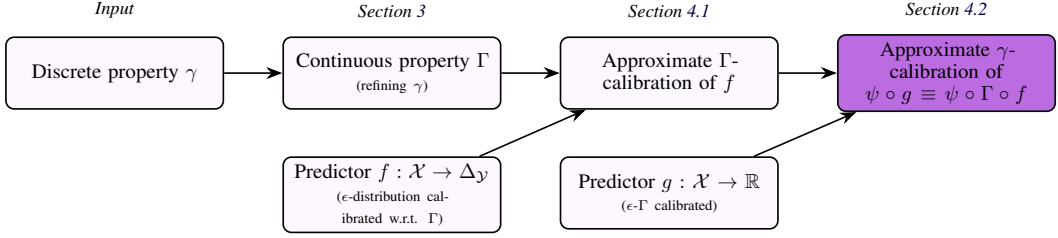

\section{Background and related work}
Consider a supervised learning task in a multiclass setting where features take values in $\X$, labels take values in a finite set $\Y$ with $3 \leq |\Y| := n < \infty$, and data is sampled i.i.d. from a measurable distribution $D \in \Delta(\X \times \Y)$.
We often contrast two predictors: a \emph{distributional} predictor $f : \X \to \simplex$ yielding distributions over labels in $\Y$, and a \emph{surrogate} predictor $g : \X \to \reals^d$ for some surrogate dimension $d$.
In particular, we often focus on scalar predictors $g : \X \to \reals$.
Let $D_{Y | \{\hat x : f(\hat x) = u\} }$ denote the marginal distribution on labels conditioned on observing the prediction $u$.
For a distribution $p \in \simplex$, we denote $p_y := \Pr[Y = y]$.

We consider loss functions measuring the error of \emph{reports} in $\R$ against \emph{ground truth labels} in $\Y$, denoting $L : \reals^d\times \Y \to \reals_+$ as a surrogate loss like hinge or cross-entropy, and $\ell : \R \times \Y \to \reals_+$ with $|\R|$ finite as a discrete loss like 0-1 loss or a Borda score.
We often represent a discrete loss $\ell$ as a cost matrix, where $\ell(r,y) = \ell_{r,y}$.
Similarly, we will discuss surrogate properties $\Gamma : \simplex \to \reals^d$, and discrete properties $\gamma : \simplex \to \R$ for a finite set $\R$.

\subsection{Properties, identifiability, embeddings}\label{subsec:elicitation}

The elicitation of discrete properties was first characterized by \citet{lambert_eliciting_2008,lambert_eliciting_2009}.
In essence, a property $\Gamma : \simplex \to \reals^d$ or $\gamma : \simplex \to \R$ maps distributions over labels to ``reports'' in some set.
The primary question driving the property elicitation research is which properties can be \emph{elicited} by specific loss functions, as elicitation tends to be a simpler tool for analysis than \emph{consistency}~\citep{bartlett_convexity_2006,zhang_statistical_2004}, which is necessary for establishing PAC learning bounds through empirical risk minimization.\footnote[2]{See works like \citep{zhang_statistical_2004,bartlett_convexity_2006,bao_calibrated_2021,bao_proper_2023,finocchiaro_unifying_2021,khurana2025consistency,ramaswamy_convex_2016,agarwal_consistent_2015,finocchiaro_embedding_2024} for a more nuanced discussion on the relationship between property elicitation, consistency, and classification calibration.}
\citeauthor{lambert_elicitation_2019} establishes a tight connection in the geometry of discrete properties and their elicitability through power diagrams \citep{aurenhammer_power_1987}; however, their characterization yields discrete losses, which are usually not tractable to directly optimize.
To circumvent this, the field has progressed in two diverging directions: restricting to the study of continuous properties~\citep{steinwart_elicitation_2014,fissler_higher_2016,noarov_statistical_2023,frongillo_elicitation_2021,peski2026nondistortionarybeliefelicitation}, or through the design of \emph{surrogate} loss functions~\citep{finocchiaro_embedding_2024,khurana2025consistency,ramaswamy_consistent_2016,ramaswamy_convex_2016,bao_calibrated_2021,wang2020weston} whose surrogate predictions in $\reals^d$ can later be \emph{linked} back to discrete decisions by a \emph{link} function $\psi : \reals^d \to \R$.
As one prominent example, the categorical cross-entropy loss elicits the identity property $\Gamma(p) = p$, which is continuous.
This continuous property is then often ``linked'' to the discrete property of the mode $\psi(u) = \argmax_y u_y$.

\begin{definition}[(Direct, indirect) elicitation, refinement]
    A loss $L : \reals^d \times \Y \to \reals_+$ \emph{(directly) elicits} a property $\Gamma : \simplex \to \reals^d$ if, for all $p \in \simplex$, we have
    \begin{align*}
        \Gamma(p) &= \argmin_{u \in \reals^d} \E_{Y \sim p} L(u,Y)~.~
    \end{align*}
    Moreover, we say $L$ \emph{indirectly elicits} a property $\gamma : \simplex \to \R$ if it elicits $\Gamma$ and there exists a link $\psi : \reals^d \to \R$ such that $r \in \Gamma(p) \implies \psi(r) \in \gamma(p)$ for all $p \in \simplex$.
    If such a link exists, we say that $\Gamma$ \emph{refines} $\gamma$.
\end{definition}
Often, we refer to the \emph{level set} or set-valued inverse of a property $\Gamma^{-1}(u) = \{ p \in \simplex : u \in \Gamma(p)\}$. 
Note that for discrete properties $\gamma : \simplex \to \R$ elicited by discrete losses $\ell : \R \times \Y \to \reals_+$, the definition of elicitation follows in the same manner, replacing $\reals^d$ with $\R$. 
However, in this case, the property might be set-valued on the boundary between two level sets $r_i$ and $r_{i+1}$, which is set of measure $0$.
In this case, we sometimes study the object $\gamma^{-1}(r_i) \cap \gamma^{-1}(r_{i+1})$ as the boundary.

\citet{frongillo_vector-valued_2015, fissler_higher_2016} and later \citet{frongillo2021elicitation} formalize the notion of \emph{elicitation complexity} for classes of loss functions or properties. 
This notion of complexity is parameterized by the dimension $d$ of inputs to the loss function.
A lower dimension $d$ can improve computational costs for gradient-based optimization methods.

\begin{definition}[Elicitation complexity {\citep{frongillo2021elicitation}}]
 For $d \in \N \cup \{\infty\}$, let $\Ek$ denote the class of all elicitable properties $\Gamma : \simplex \to \reals^d$.
    Let $\C$ be a class of properties and $d \in \mathbb{N} \cup \{\infty\}$.
    A property $\Gamma : \simplex \to \R$ is $d$-elicitable with respect to $\C$ if there exists an intermediate property $\hat \Gamma \in \C \cap \Ek$ and a map $\psi$ such that $\Gamma = \psi \circ \hat \Gamma$.
    The elicitation complexity of $\Gamma$ is $\elic_\C(\Gamma) = \min\{d: \Gamma \text{ is } d-\text{elicitable with respect to } \C\}$.
\end{definition}

For example, while categorical cross-entropy takes $d = n$-dimensional inputs--- the highest sensible dimension possible, the identity property it elicits $\Gamma(p) = p \in \reals^n$ can be post-processed into \emph{any} discrete property $\gamma : \simplex \to \R$, even those which are not elicitable themselves, providing a trivial upper bound on elicitation complexity of $n$ for all properties $\gamma$.

We are particularly interested in the \emph{Lipschitz continuous} elicitation complexity of discrete properties $\gamma$, where a property $\Gamma : \simplex \to \reals^d$ is \emph{Lipschitz}.
Let $\CLip{K}$ denote the class of $K$-Lipschitz properties, and $\CLipnoK$ denote the set of properties which are $K$-Lipschitz for some $K \geq 0$.
Without loss of generality, we assume $K$-Lipschitz property values are bounded in the relevant $[0,1]^d$ hypercube.

Characterizations of these notions of complexity are most complete in $1$ dimension, and are established in the context of \emph{orderable} properties~\citet{finocchiaro_embedding_2020,khurana2025consistency}.
We pose a notion of \emph{strong orderability} necessary for obtaining Lipschitz properties.
With these properties, we construct differentiable surrogates $L : \reals \times \Y \to \reals_+$ indirectly eliciting strongly orderable properties; the intermediate property directly elicited by $L$ is in $\CLipnoK$, which we use in \Cref{sec:calibration} for approximate calibration guarantees. 

\begin{definition}[{(Strongly) Orderable property~\citep{finocchiaro_embedding_2020}}]
    A finite property $\gamma : \simplex \to 2^\R \setminus \{\emptyset\}$ is \emph{orderable} if there is an enumeration of $\R = \{r_1, \ldots, r_{|\R|}\}$ such that for all $i \leq |\R| - 1$, we have $\gamma^{-1}(r_i) \cap \gamma^{-1}(r_{i+1})$ is a hyperplane intersected with $\simplex$.
    Moreover, we say $\gamma$ is \emph{strongly orderable} if it is orderable and, for all $i \in \{2, \ldots, |\R| - 1\}$, we have $\inf_{p \in \gamma^{-1}(r_{i-1}) \cap \gamma^{-1}(r_i), q \in \gamma^{-1}(r_i)\cap \gamma^{-1}(r_{i+1})}\|p - q\|$ is bounded away from $0$.
\end{definition}

\subsection{Calibration in multiclass settings}
While calibration of predictors has been well-studied in binary outcome settings, definitions in multiclass settings are not universally agreed upon, and results are generally definition-specific.
Recently, \citet{derr2025threetypesofcalibration} categorize three primary approaches to defining calibration in multiclass settings, though we focus on two: \emph{distribution calibration with respect to $\gamma$} and \emph{$\Gamma$-calibration}.
For intuition, consider calibration error which measures (a) conditioned on a piece of information, such as a model's prediction, the (b) distance between some prediction against a baseline.
The computational and statistical challenge posed by na\"ive extensions of multiclass calibration is that the granularity of provided conditioning information in (a) grows exponentially in the number of classes $n$.

\citeauthor{derr2025threetypesofcalibration} characterize the first set of multiclass calibration definitions as \emph{distribution calibration with respect to $\gamma$}, where $\gamma : \simplex \to \R$ is some given property representing the conditioning event (a), with granularity $|\R| = \mathbf{poly}(n)$~\citep{zhao_calibrating_2021,bairaktari2025sample,gopalan_computationally_2024}.

\begin{definition}[Approximate distribution calibration with respect to $\gamma$]
    A distributional predictor $f : \X \to \simplex$ is $\epsilon$-\emph{distribution calibrated with respect to $\gamma$} if
    \begin{align*}
        \E_{X, Y} \| f(X) - D_{Y \mid \{\hat x : \gamma(f(\hat x)) = \gamma(f(X))\}} \| & \leq \epsilon~.
    \end{align*}
\end{definition}
Of course, the choice of norm affects the meaning of a calibration bound; this is demonstrated by considering the diameter of the simplex for different $p$-norms: with $p = 1$, the simplex has diameter $2$; with $p = 2$, the diameter is $\sqrt{2}$, and with $p = \infty$, the diameter is $1$.
Generally speaking, we use an unspecified $L_p$ norm for $p \geq 1$ following the precedent of \citet{garg_oracle_2023}.
\citet[Definition 2.3]{garg_oracle_2023} discusses the most common metric choices being $L_1$, $L_2$, and $L_\infty$ norms, and proceed to give results relating approximate calibration coefficients with different choices of $L_p$-norms.
Relationships with other calibration error metrics are discussed in \Cref{app:calibration-metrics}.

Moving one step beyond distribution calibration with respect to $\gamma$, $\Gamma$-calibration focuses on the distance of predictions and baseline in (b), measuring distance in \emph{prediction space} $\im(\Gamma)$ rather than distribution space $\simplex$.
For discrete predictions and properties, this distance is not continuous in function space, and approximate calibration is not empirically stable.
To sidestep this issue, \citet{noarov_statistical_2023} and \citet{gneiting_regression_2023} assume properties $\Gamma$ are continuous, \citet{derr2025threetypesofcalibration} only give \emph{exact} $\Gamma$-calibration results for discrete properties, \citet{hu2025efficientswapmulticalibrationelicitable} focus on binary outcome settings (emphasizing connections to calibration error metrics), and \citet{collina2026samplecomplexitymulticalibration} study the sample complexity of multicalibration for continuous scalar properties.

\citet{derr2025threetypesofcalibration} discuss approximate $\Gamma$-calibration in terms of an arbitrary metric, but do not give significant approximate results for discrete properties.
With this insight, we use continuous properties so that we can use $L_p$-norms as a distance metric, and bound the approximate calibration of discrete properties through 0-1 error.

\begin{definition}[Approximate $\Gamma$-calibration for continuous properties]
    Consider a continuous property $\Gamma : \simplex \to \reals^d$ and $\Gamma$-predictor $g : \X \to \reals^d$. 
    $g$ is $\epsilon$-approximately $\Gamma$-calibrated if,  $\E_{X,Y}\|\Gamma(D_{Y|\{ \hat x : g(\hat x) = g(X)\}}) - g(X)\| \leq \epsilon$.
\end{definition}

When defining a miscalibration metric for discrete properties, we use the 0-1 metric $m(s, r) = \mathbbm{1}(s \neq r)$.
The discrete predictor in the following definition $h : \X \to \R$ is often not directly optimzed, but formed as a post-processing of some continuous predictor.
That is, we often have $h = \psi \circ g$ for a link function $\psi : \reals^d \to \R$.

\begin{definition}[Approximate $\gamma$-calibration for discrete properties]
    Consider a discrete property $\gamma : \simplex \to \R$ for a finite set $\R$ and $\gamma$-predictor $h : \X \to \R$. 
    $h$ is $\epsilon$-approximately $\gamma$-calibrated if  $Pr_{X}[\gamma(D_{Y|h(X)}) \neq h(X)] \leq \epsilon$.
\end{definition}

With preliminaries established, we now restate our goal: we aim to understand the relationship between the approximate distribution calibration of a predictor $f : \X \to \simplex$ and the approximate $\gamma$-calibration of its post-processing $\psi \circ \Gamma \circ f$, where $\Gamma$ is a continuous property refining $\gamma$.
We now proceed by constructing $\Gamma$ for strongly orderable $\gamma$.

\section{Characterizing 1-dimensional Lipschitz elicitable properties}\label{sec:elic-char}
In this section, we establish that orderable properties are refined by $1$-dimensional Lipschitz properties $\Gamma : \simplex \to \reals$, and as a consequence, we conclude that for any strongly orderable property $\gamma$, $\elic_{\CLipnoK}(\gamma) = 1$.
We later apply this characterization to establish approximate calibration guarantees for such discrete orderable properties.
Notably, our characterization is constructive, as \Cref{alg:smooth-embedding} and \Cref{alg:piecewise-RoE-normals} yield Lipschitz properties $\Gamma$ refining $\gamma$.
\Cref{alg:smooth-embedding} modifies the (piecewise linear, convex, but not differentiable) surrogate construction of \citet[Theorem 11]{finocchiaro_embedding_2020} and interpolates to ``smooth'' their surrogate; such an interpolation is not straightforward to ensure property refinement.
This approach is also very similar to the approach of \citet[Appendix H]{khurana2025consistency}, though they do not establish Lipschitzness in their work.
We conjecture this approach might be generalized to higher dimensions, but leave this to future work.

In contrast, \Cref{alg:piecewise-RoE-normals} uses the geometry of the discrete property $\gamma$ directly, constructing a property $\Gamma$ which is a piecewise ratio of expectations, where the functions forming the numerator and denominator are defined by the hyperplanes defining an orderable property.
Crucially, it is not clear if it possible to generalize this approach for $d > 1$.
With the Lipschitz property $\Gamma$ in both algorithms, we apply the characterizations of \citet{steinwart_elicitation_2014} and \citet{finocchiaro_convex_2018} to construct a surrogate loss indirectly eliciting $\gamma$. 

Both \Cref{alg:smooth-embedding} and \Cref{alg:piecewise-RoE-normals} return surrogate losses eliciting Lipschitz continuous properties, but these properties are not necessarily bounded in $[0,1]^d$.
They are however, generally bounded, and therefore can be scaled by $L^*(u, y) := L(\frac{u - \Gamma_{\min}}{\Gamma_{\max} - \Gamma_{\min}}, y)$.
For simplicity, we just return $L$.

\paragraph{Motivating example}
Consider the loss matrix $\ell$\footnote[2]{Recall $\ell_{r,y} = \ell(r,y)$ is the discrete loss with report $r$ and outcome $y$.} and elicited discrete property $\gamma$ as follows, with outcomes labeled $\Y = \{1,2,3\}$.
\begin{align}
    \ell &= \begin{bmatrix}
        0 & 3 & 5 \\ 1 & 0 & 3 \\ 3 & 1 & 0
    \end{bmatrix}\label{eq:motivating-example}~,~ & 
    \gamma(p) &= \begin{cases}
    1 & -3 p_1 + p_2 \leq -2 \\ 
2 & -3p_1 + p_2 \geq -2 \, \wedge \,  5p_1 + 4p_2 \ge 3 \\
3 & 5p_1 + 4p_2 \le 3
\end{cases}
\end{align}
While $\ell$ and $\gamma$ are discrete, we seek a continuous relaxation $\Gamma$ of $\gamma$ in order to show $\elic_\CLipnoK(\gamma) = 1$.
\Cref{fig:asymmetric-examples} depicts the level sets $\Gamma^{-1}(u)$ of the continuous properties yielded from \Cref{alg:smooth-embedding} and \Cref{alg:piecewise-RoE-normals}, respectively.
Applications of our algorithms to this example are in \Cref{app:apply-algos}.

\begin{figure}[htbp]
     \centering
     \begin{subfigure}[b]{0.35\textwidth}
         \centering
         \resizebox{\linewidth}{!}{\definecolor{purple}{RGB}{209, 120, 250}

\begin{tikzpicture}[scale=8]
    \coordinate (A) at (0, 0); 
    \coordinate (B) at (1, 0); 
    \coordinate (C) at (0.5, {sqrt(3)/2});

    \coordinate (B12_AC) at (0.125, {0.25*sqrt(3)/2}); 
    \coordinate (B12_AB) at (0.333, 0);               
    
    \coordinate (B23_AB) at ({4/11}, 0);              
    
    \coordinate (B34_BC) at (0.625, {0.75*sqrt(3)/2}); 
    \coordinate (B34_AB) at (0.4, 0);                  

    \fill[blue!10] (A) -- (B12_AC) -- (B12_AB) -- cycle;
    \fill[green!10] (B12_AC) -- (C) -- (B23_AB) -- (B12_AB) -- cycle;
    \fill[orange!10] (C) -- (B34_BC) -- (B34_AB) -- (B23_AB) -- cycle;
    \fill[red!10] (B34_BC) -- (B) -- (B34_AB) -- cycle;

    \foreach \c in {0.05, 0.1,..., 0.5} {
        \pgfmathsetmacro{\pOne}{4*\c / (6 + 4*\c)}
        \pgfmathsetmacro{\pTwo}{4*\c / (5 + 2*\c)}
        \begin{scope}
            \clip (A) -- (B12_AC) -- (B12_AB) -- cycle;
            \draw[blue!70!black, thin] ({0.5*\pOne}, {\pOne*sqrt(3)/2}) -- (\pTwo, 0);
        \end{scope}
    }

    \foreach \c in {.55, .6, ..., 1.0} {
        \pgfmathsetmacro{\pOne}{4 / (28 - 24*\c)}
        \pgfmathsetmacro{\pTwo}{4 / (13 - 2*\c)}
        \begin{scope}
            \clip (B12_AC) -- (C) -- (B23_AB) -- (B12_AB) -- cycle;
            \draw[green!60!black, thin] ({0.5*\pOne}, {\pOne*sqrt(3)/2}) -- (\pTwo, 0);
        \end{scope}
    }

    \foreach \c in {1.05, 1.1, ..., 2.0} {
        \pgfmathsetmacro{\pOne}{4 / (6 - 2*\c)}
        \pgfmathsetmacro{\pTwo}{4 / (12 - \c)}
        \begin{scope}
            \clip (C) -- (B34_BC) -- (B34_AB) -- (B23_AB) -- cycle;
            \draw[orange!80!black, thin] ({0.5*\pOne}, {\pOne*sqrt(3)/2}) -- (\pTwo, 0);
        \end{scope}
    }

    \foreach \c in {2.05, 2.1, ..., 3.0} {
        \pgfmathsetmacro{\pOne}{(4*\c - 4) / (4 - \c)}
        \pgfmathsetmacro{\pTwo}{(4*\c - 4) / (14 - 2*\c)}
        \begin{scope}
            \clip (B34_BC) -- (B) -- (B34_AB) -- cycle;
            \draw[red!70!black, thin] ({0.5*\pOne}, {\pOne*sqrt(3)/2}) -- (\pTwo, 0);
        \end{scope}
    }

    \draw[black, thick] (B12_AC) -- (B12_AB) node[left, pos=0.2] {\scriptsize $u=\frac 12$};
    \draw[black, thick] (C) -- (B23_AB) node[left, pos=0.6] {\scriptsize $u=\frac 3 2$};
    \draw[black, thick] (B34_BC) -- (B34_AB) node[right, pos=0.6] {\scriptsize $u=2$};

    \draw[very thick] (A) -- (B) -- (C) -- cycle;
    \node[below left] at (A) {$Y = 1$};
    \node[below right] at (B) {$Y = 3$};
    \node[above] at (C) {$Y = 2$};
\end{tikzpicture}}
         \caption{Level sets $\Gamma^{-1}(u)$ of the property $\Gamma$ constructed in \Cref{alg:smooth-embedding}.}
     \end{subfigure}
     \hfill
     \begin{subfigure}[b]{0.35\textwidth}
         \centering
         \resizebox{\linewidth}{!}{\definecolor{purple}{RGB}{209, 120, 250}

\begin{tikzpicture}[scale=8]
    \coordinate (A) at (0, 0); 
    \coordinate (B) at (1, 0); 
    \coordinate (C) at (0.5, {sqrt(3)/2});

    \coordinate (B12_AC) at (0.125, {0.25*sqrt(3)/2}); 
    \coordinate (B12_AB) at (0.333, 0);               
    
    \coordinate (B23_AB) at ({4/11}, 0);              
    
    \coordinate (B34_BC) at (0.625, {0.75*sqrt(3)/2}); 
    \coordinate (B34_AB) at (0.4, 0);                  

    \fill[blue!10] (A) -- (B12_AC) -- (B12_AB) -- cycle;
    \fill[green!10] (B12_AC) -- (C) -- (B34_BC) -- (B34_AB) -- (B12_AB) -- cycle;
    \fill[orange!10] (B34_BC) -- (B) -- (B34_AB) -- cycle;

    \foreach \c in {-.95, -.9,..., 0} {
        \pgfmathsetmacro{\pOne}{(\c*3.74165 + 1)/4}
        \pgfmathsetmacro{\pTwo}{(\c*3.74165 + 1)/3}
        \begin{scope}
            \clip (A) -- (B12_AC) -- (B12_AB) -- cycle;
            \draw[blue!70!black, thin] ({0.5*\pOne}, {\pOne*sqrt(3)/2}) -- (\pTwo, 0);
        \end{scope}
    }

    \foreach \c in {.05, .1, ..., 1} {
        \pgfmathsetmacro{\pOne}{(\c + 1)/(4 - 3*\c)}
        \pgfmathsetmacro{\pTwo}{(\c + 1)/(3 + 2*\c)}
        \begin{scope}
            \clip (B12_AC) -- (C) -- (B34_BC) -- (B34_AB) -- (B12_AB) -- cycle;
            \draw[green!60!black, thin] ({0.5*\pOne}, {\pOne*sqrt(3)/2}) -- (\pTwo, 0);
        \end{scope}
    }

    \foreach \c in {1.05, 1.1, ..., 2.0} {
        \pgfmathsetmacro{\pOne}{((\c-1)*3.74165 + 2)}
        \pgfmathsetmacro{\pTwo}{((\c-1)*3.74165 + 2)/5}
        \begin{scope}
            \clip (B34_BC) -- (B) -- (B34_AB) -- cycle;
            \draw[orange!70!black, thin] ({0.5*\pOne}, {\pOne*sqrt(3)/2}) -- (\pTwo, 0);
        \end{scope}
    }

    \draw[black, thick] (B12_AC) -- (B12_AB) node[left, pos=0.2] {\scriptsize $u=0$};
    \draw[black, thick] (B34_BC) -- (B34_AB) node[right, pos=0.6] {\scriptsize $u=1$};

    \draw[very thick] (A) -- (B) -- (C) -- cycle;
    \node[below left] at (A) {$Y = 1$};
    \node[below right] at (B) {$Y = 3$};
    \node[above] at (C) {$Y = 2$};
\end{tikzpicture}}
         \caption{Level sets of the property constructed in \Cref{alg:piecewise-RoE-normals}.}
     \end{subfigure}
     \hfill
     \begin{subfigure}[b]{0.2\linewidth}
         \caption{Comparison of the continuous properties produced from \Cref{alg:smooth-embedding} and \Cref{alg:piecewise-RoE-normals}, respectively, applied to \Cref{eq:motivating-example} The black contours depict the property $\gamma$ in \eqref{eq:motivating-example}, or a slight refinement of it in the case of the embedding.}
     \label{fig:asymmetric-examples}
     \end{subfigure}
\end{figure}

\subsection{Algorithm from embeddings}
\Cref{alg:smooth-embedding} uses the \emph{embeddings} approach of \citet{finocchiaro_embedding_2024}.
Given a discrete loss $\ell : \R \times \Y \to \reals_+$, \citeauthor{finocchiaro_embedding_2024} provide a construction of a piecewise linear and convex surrogate $L : \reals^{n-1} \times \Y \to \reals_+$ by ``embedding'' discrete reports into a continuous space for optimization through the function $\varphi : \R \to \reals^{n-1}$.
When properties are \emph{orderable}, \citet[Theorem 11]{finocchiaro_embedding_2020} produces a 1-dimensional surrogate loss $L : \reals \times \Y \to \reals_+$ and embedding $\varphi : \R \to \reals$ so that $L$ indirectly elicits $\gamma$ elicited by $\ell$.
However, since their embeddings are piecewise linear, they do not elicit a Lipschitz property.
We must take care to ``smooth'' these embeddings while maintaining property refinement guarantees.

\Cref{alg:smooth-embedding} is very similar to the construction given in \citet[Appendix H]{khurana2025consistency}.
However, their construction linearly interpolates between loss values between integers in $[k]$.
By adding midpoints to the interpolation set, we are able to verify refinement in \Cref{alg:smooth-embedding} as level sets of $\Gamma$ coincide with boundaries of $\gamma$ at these midpoints. 
We let $\delta^+ L(u, \cdot)$ denote the right directional derivative with respect to the first argument at $u$, and likewise for $\delta^-$ as the left derivative.

\begin{algorithm}
    \begin{algorithmic}
        \Procedure {Smooth Embedding}{Orderable $\gamma$}
        \State Construct embedding $L : \mathbb{R} \times \Y \to \mathbb{R}_+$ and embedding function $\varphi: \R \to \reals$.
        \State Take interpolation set $U = \varphi(\R) \cup \{\frac{\varphi(r_i) + \varphi(r_{i+1})}{2} : i \in [|\R|-1] \}$
        \State Take function $V(u,y) = \begin{cases}
            \frac{d}{du} L(u,y) & \text{$L$ differentiable} \\
            0 & \text{sign($\delta^+(L(u,y))$)} \neq \text{sign($\delta^- (L(u, y))$)} \\
            \frac{\delta^+ L(u,y) + \delta^- L(u,y)}{2} & \text{otherwise}
        \end{cases}$
        \State Construct $\bar V$ by linearly interpolating between $\{V(u, \cdot) : u \in U \}$ on $\conv(U)$, and matching $V$ outside $\conv(U)$. 
        \State Integrate $\bar V(u,y)$ to get $\bar L(u,y) = \int_0^u\bar V(z,y) dz$
        \State Construct $\psi: \mathbb{R} \to \R$ as a rounding based on the interpolation set $U$.
        \State \Return $\bar L, \psi$.
        \EndProcedure
    \end{algorithmic}
    \caption{Algorithm for generating continuous property from discrete embedding}\label{alg:smooth-embedding}
\end{algorithm}

The loss function constructed in \Cref{alg:smooth-embedding} is piecewise quadratic as the integration of piecewise linear functions.
This construction is not unique, as exemplified by the scoring rule characterization of \citet{steinwart_elicitation_2014} and \citet{finocchiaro_convex_2018}, but one benefit of this form is that the derived property $\Gamma$ is Lipschitz continuous with constant $\max_{u,y} |\bar V(u,y)|$; since $\bar V$ is bounded, this constant is bounded.

\begin{propositionE}[][end, restate]\label{thm:1-embeddable-iff-1Lip}
    Given a strongly orderable property $\gamma$, \Cref{alg:smooth-embedding} returns a surrogate loss $\bar L$ and link $\psi$ such that $\bar L$ elicits a Lipschitz continuous property $\Gamma$ refining $\gamma$ through the link $\psi$. 
    Therefore, $\elic_\CLipnoK(\gamma) = 1$.
\end{propositionE}
\begin{proofE}
    We show (a) a closed form of $\Gamma$ elicited by $\bar L$ constructed in \Cref{alg:smooth-embedding}, (b) $\Gamma$ is Lipschitz, and (c) $\Gamma$ refines $\gamma$.

    \paragraph{(a)}

    The spirit of \Cref{alg:smooth-embedding} is to ``smooth'' a non-smooth loss into a smooth one by interpolating a step function $V$ into a piecewise linear function $\bar V$.
    \citet[Theorem 5]{steinwart_elicitation_2014} establishes that if a continuous function $\bar V(u, y)$ yields loss $\bar L(u,y) := \int_0^u \bar V(z,y) dz$, then $\bar L$ elicits the property $\Gamma$, $\Gamma$ is the root of $\E_{Y \sim p} \bar V(\cdot, Y)$.
    That is, $\E_{Y \sim p}\bar V(u, Y) = 0\iff u = \Gamma(p)$.
    Therefore, we can derive $\Gamma(p)$ by examining the root of $\bar V$.
    Since $\bar V$ is piecewise linear, we examine each piece $j$ independently for its root:
    \begin{align*}
        \E_{Y \sim p} \bar V(u, Y) = 0 &\iff \sum_y p_y(a^{(j)}_y u + b^{(j)}_y) = 0 
        \iff u =  -\frac{\inprod{p}{b^{(j)}}}{\inprod{p}{ a^{(j)}}} ~.
    \end{align*}
    
    Writing $\bar V(u, y)$ as a linear interpolation, we obtain $\bar V(u,y) = V(s_{i-1}, y) + (u - s_{i-1}) \left( \frac{V(s_i, y) - V(s_{i-1}, y)}{s_i - s_{i-1}}\right)$ for some $s_{i-1}, s_i \in U$, which enables us to re-write 
    \begin{align*}
        \Gamma(p) &= - \frac{\E_{Y \sim p} V(s_{j(p)-1}, Y) - s_{j(p)-1} \frac{\E_{Y \sim p} V(s_{j(p)}, Y) - \E_{Y \sim p} V(s_{j(p)-1}, Y)}{s_{j(p)} - s_{j(p)-1}} }{\frac{\E_{Y \sim p} V(s_{j(p)}, Y) - \E_{Y \sim p} V(s_{j(p)-1}, Y)}{s_{j(p)} - s_{j(p)-1}}}~,
    \end{align*}
    where $j(p) = \min\{l : \E_{Y \sim p} V(U_l, Y) > 0\}$, where $U_l$ is the $l^{th}$ ordered element of $U$ (and $\min(\emptyset) := 0$).
    Crucially, this implies that the level sets defining cases of $\Gamma(p)$ are contained in $U$, as $j(p)$ increments.

    \paragraph{(b)}
    Notably, $\Gamma$ is Lipschitz continuous as long as $\inprod{p}{a^{(j)}}$ is bounded away from $0$ for all pieces $j$. 
    If this value did approach $0$ for some $p$, then the identification function becomes flat on an interval of nonzero (Lebesgue) measure, meaning the hyperplanes separating discrete level sets are not bounded away from each other, which contradicts strong orderability of $\gamma$.

    \paragraph{(c)}
    $\Gamma$ is a piecewise ratio of expectations, where pieces correspond to the interpolation set.
    We claim that, for each $i \in \{1, 2, \ldots, k-1\}$, there is an element $s \in U$ of the interpolation set such that $\Gamma^{-1}(s) = \gamma^{-1}(r_i) \cap \gamma^{-1}(r_{i+1})$.
    To construct the link function, it helps to decompose $U = \varphi(\R) \cup S$, where $S = \{s_1, \ldots, s_{k-1}\}$ is the constructed set of midpoints. 
    Consider the link function $\psi(u) = r_k$, with $k = \#\{ j : s_j < u \} +1$.

    Since $\psi(u)$ only increments at elements of $S$, $j(p)$ only increments at elements of $U$, and elements of $U$ are ordered alternating embedding points and midpoints (in $S$), $\psi \circ \Gamma$ can be rewritten $p \mapsto \lceil \frac{j(p)}{2} \rceil$, which aggregates level sets at the interpolation points in $S$.
    Therefore, $\Gamma$ refines $\gamma$.

\end{proofE}

\subsection{Algorithm from normals}
We now present another construction for constructing Lipschitz continuous refinements of orderable properties that directly leverages the geometric structure of the property itself, instead of ``smoothing'' a non-smooth surrogate.
In essence, \Cref{alg:piecewise-RoE-normals} constructs a surrogate property $\Gamma$ as a piecewise ratio of expectations of two linear functions, which are inner products of vectors normal to the boundaries of level sets.
Strong orderability of $\gamma$ is necessary here since it enables us to construct precisely one normal vector per boundary.
We say a set of vectors $o_1, \ldots, o_{k-1}$ defining level sets $\gamma^{-1}(r_1), \ldots, \gamma^{-1}(r_k)$ is \emph{oriented} if $p \in \gamma^{-1}(r_i) \iff \inprod{o_{i+1}}{p} \leq 0 \leq \inprod{o_i}{p}$ for all $i \in 1, \ldots, k$, omitting the undefined inequalities for $i \in\{1, k\}$.

\begin{algorithm}
    \begin{algorithmic}
        \Procedure {Generate Piecewise RoE}{Orderable $\gamma$}
        \ForAll{$i \in 1, \ldots, |\R| - 1 := k$}
        \State Find $n-1$ distributions $p_{ij} \in \gamma^{-1}(r_i) \cap \gamma^{-1}(r_{i+1}) \cap \relint(\simplex)$ to form $P_i \in \reals^{n-1 \times n}$, and define a normal $o_i \in \mathrm{ker}(P_i)$ using SVD.
        \State Ensure $o_i$ is oriented
        \EndFor
        \State Construct identification function
        \State $V(u,y) = u - \chi(u, [0, |\R| - 2]) - o_{1,y} - \sum_{j=1}^{k-1} \chi(u-(j-1), [0,1]) (o_{j+1,y} - o_{j,y})$
        \State $L(u,y) = \int_0^u V(x, y) dx$
        \State $\psi: x \mapsto \chi(\lceil x \rceil, [0, |\R|-1]) + 1$
        \State \Return $L, \psi$
        \EndProcedure
    \end{algorithmic}
    \caption{Algorithm for generating Piecewise ROE from normals, where $\chi(x, I)$ is the clipping of $x$ to interval $I$.}\label{alg:piecewise-RoE-normals}
\end{algorithm}

\begin{propositionE}[][end,restate]\label{prop:normals-alg-gives-refining-prop}
    Let $\gamma$ be a strongly orderable finite property.
    Then \Cref{alg:piecewise-RoE-normals} yields a surrogate loss $L$ and link $\psi$ such that $L$ elicits a Lipschitz continuous property $\Gamma : \simplex \to \reals$ which refines $\gamma$ through the link $\psi$.
\end{propositionE}
\begin{proofE}
    Without loss of generality, define $\gamma(p) = \min\{i \in [k-1] : \inprod{o_i}p \geq 0$\}.
    If this set is empty, define $\gamma(p) = 0$.
    
    We aim to (a) derive a closed form of the property $\Gamma$ elicited by the loss returned in \Cref{alg:piecewise-RoE-normals},  (b) $\Gamma$ is continuous in $u$ for all $y \in \Y$, and (c) there is a link function $\psi : \reals \to \{1, \ldots, k\}$ such that $\psi(\Gamma(p)) = \gamma(p)$. 
    If all of these statements hold, then $\Gamma$ refines $\gamma$ by the link function $\psi$.

    \paragraph{(a) Closed form of $\Gamma$} Consider $3$ cases for the construction of $V$:
    \begin{align*}
        V(u,y) &= \begin{cases}
            u -o_{1,y} & u \leq 0 \\
            -o_{i,y} - (u-(i-1))(o_{i+1,y}- o_{i,y}) & u \in [0, k-1] \\
            u - (k-1) -o_{k,y} & u \geq k-1
        \end{cases}~,~
    \end{align*}
    where the intermediate case is a simplification arising from a telescoping sum.
    For each of these cases, we use the scoring rule characterization of \citet{steinwart_elicitation_2014} and  set $\E_{Y \sim p} V(u,Y) = 0 \iff u = \Gamma(p)$ to obtain
    \begin{align*}
        \Gamma(p) = \begin{cases}
        \inprod {o_1}{p} & \inprod{o_1}{p} \leq 0 \\
        \frac{\inprod{o_i}{p}}{\inprod{o_{i} - o_{i+1}}{p}} + (i - 1) & \exists i \in [k-1] : \inprod{o_{i+1}}{p} \leq 0 \leq \inprod{o_i}{p} \\
        \inprod{o_k}{p} + (k-1) & \inprod{o_k}p \geq 0
    \end{cases}~.~
    \end{align*}
    Since the normals $\vec o$ are oriented, we conclude that every $p \in \simplex$ satisfies one of these criteria. 
    Moreover, on the boundaries, we can see that $\Gamma$ is continuous the boundary criterion of $\inprod{o_i} p = 0$ implies $\Gamma(p) = (i-1)$, while the boundary criterion of $\inprod{o_{i+1}} p = 0$ implies $\Gamma(p) = 1 + (i-1) = i$.

    \paragraph{(b) $\Gamma$ is Lipschitz}
    Observe that a ratio of linear functions is Lipschitz continuous as long as the numerator and denominator are bounded, and the denominator is bounded away from $0$.
    Since the domain of $\gamma$ is the simplex, and is therefore bounded, so are the numerator and denominator.
    Moreover, the denominator is bounded away from $0$ if all pairs of hyperplanes forming boundaries of $\gamma$ are bounded away from each other, which follows from strong orderability of $\gamma$.

    \paragraph{(c)} Finally, consider the link $\psi: x \mapsto \chi(\lceil x \rceil, [0, |\R|-1]) + 1$.
    If $\Gamma(p) < 0$, then $\psi(\Gamma(p)) = 1$, for $\Gamma(p) \in [1, k]$, we have $\psi(\Gamma(p)) = \lceil\Gamma(p)\rceil$, and for $\Gamma(p) \geq k = |\R| - 1$, we have $\psi(\Gamma(p)) = |\R|$.
    Therefore, for all $p \in \simplex$, $\Gamma(p) < 0 \implies \inprod{o_1}p < 0 \implies \gamma(p) = 1 = \chi(\lceil (\Gamma(p)\rceil, [0, k-1]) +1$.
    Similarly, $\Gamma(p) \in [i-1, i)\implies  \inprod{o_{i}}p \geq 0 \geq \inprod{o_{i+1}}{p} \implies \gamma(p) = i + 1 = \lceil \Gamma(p)\rceil = \chi(\lceil \Gamma(p)\rceil, [0, k-1]) + 1$ for any $i \in [k]$.
    Finally, $\Gamma(p) \geq k \implies \inprod{o_{k}}p \geq 0 \implies \gamma(p) = k = \chi(\lceil (\Gamma(p)), [0, k-1]) + 1$.
    Therefore, $\Gamma$ refines $\gamma$ by the link $\psi$.
\end{proofE}

\section{Application to $\Gamma$-calibration}\label{sec:calibration}
The characterization of Lipschitz elicitation complexity, particularly with the construction of Lipschitz refining properties in \Cref{alg:smooth-embedding,alg:piecewise-RoE-normals}, enables us to establish bounds on the approximate $\Gamma$-calibration of both distributional predictors post-processed by direct computation of $\Gamma$ (\Cref{subsec:dist-Gamma}), and from $\Gamma$-predictors to discrete $\gamma$-predictors (\Cref{subsec:Gamma-gamma}).
The results of these sections can be composed to obtain bounds on the probability of miscalibration with respect to a discrete property when given a distributional predictor $f : \X \to \simplex$.

\subsection{Relationship between Distribution calibration and $\Gamma$-calibration}\label{subsec:dist-Gamma}

\begin{theorem}[$\Gamma$-calibration bounds by post-processing distributional predictor]
    Let $\Gamma : \simplex \to \reals$ be K-Lipschitz.
    Then a predictor $f$ that is $\epsilon$-distribution calibrated with respect to $\Gamma$ yields predictor $\Gamma \circ f$ which is $K\epsilon$-approximately $\Gamma$-calibrated.
\end{theorem}
\begin{proof}
    \begin{align*}
    \E_{X,Y} \|\Gamma(f(X)) - \Gamma(D_{Y |\{\hat x : \Gamma(f(\hat x)) = \Gamma(f(X))\}})\| &\leq \E_{X,Y} [K\|f(X) - D_{Y |\{\hat x : \Gamma(f(\hat x)) = \Gamma(f(X))\}}\| ] \\
    &= K \E_{X,Y} \|f(X) - D_{Y |\{\hat x : \Gamma(f(\hat x)) = \Gamma(f(X))\}}\| \\
    &\leq K \epsilon
    \end{align*}
    Observe the first inequality follows in expectation since the Lipschitz inequality holds pointwise for all $p \in \simplex$. 
\end{proof}
This result is notably a tighter bound than that of \citet[Proposition 7]{derr2025threetypesofcalibration}, which gives a bound of $K \epsilon |\Y|$ because they allow the metric measuring distance to be arbitrary.
They establish their bound using total variation distance rather than pulling the Lipschitz constant out using properties of the expectation function.

As one important corollary, we observe that if the smoothed property is a contraction mapping (it is Lipschitz with constant $K < 1$), then the calibration error bound is unchanged.

\begin{corollary}\label{cor:contraction-preserves-bound-exactly}
    Let $\Gamma : \simplex \to \reals$ be a contraction mapping.
    Then a predictor $f$ that is $\epsilon$-distribution calibrated with respect to $\Gamma$ yields predictor $\Gamma \circ f$ which is $\epsilon$-$\Gamma$-calibrated.
\end{corollary}

For properties $\Gamma$ that are not contraction mappings, small distribution calibration with respect to $\Gamma$ might be deceiving, as small calibration error can lead to very bad decisions with respect to $\Gamma$.
As a lower bound, we construct a counterexample of a distributional predictor $f : \X \to \simplex$ where $f$ is $\epsilon$-distribution calibrated, but \emph{not} $\epsilon$-approximately $\Gamma$ calibrated, despite $\im(\Gamma)$ being lower-dimensional than $\im(f)$.

\begin{proposition}\label{prop:Lipschitz-lb}
    Let $\Gamma : \simplex \to \reals^d$ be a Lipschitz property with optimal constant $K = \inf\{k \geq 0 : \|\Gamma(p) - \Gamma(q)\| \leq k \|p - q\| \forall p,q \in \simplex \}$.
    Then for every $C < K$, there exists a data distribution over $\Delta(\X \times \Y)$ and distributional predictor $f : \X \to \simplex$ such that $\E_{X,Y} \|f(X) - D_{Y |\{\hat x : \Gamma(f(\hat x)) = \Gamma(f(X))\}}\| = \epsilon$, yet $f$ fails $C\epsilon$-approximate $\Gamma$ calibration.
    That is, $\E_{X,Y} \|\Gamma(f(X)) - \Gamma(D_{Y |\{\hat x : \Gamma(f(\hat x)) = \Gamma(f(X))\}})\| > C\epsilon$.
\end{proposition}
\begin{proof}
    We show this by constructing a counterexample. 
    For simplicity, we consider $\X = \{x\}$.
    Fix $C < K$.
    There must be some $p, q$ such that $\|\Gamma(p) - \Gamma(q)\| > C \|p - q\|$, otherwise $K$ would not be the optimal Lipschitz constant.
    Consider predictor $f(x) = p$ and distribution $D$ such that $D_{Y | \Gamma(f(x))} = D_Y = q$.
    Now consider $\epsilon := \| p - q\| = \E_{Y} \|f(x) - D_{Y}\| = \E_{X,Y} \|f(X) - D_{Y | \{\hat x : \Gamma(f(\hat x)) = \Gamma(f(X))\}}\|$.
    This strategic choice of $\epsilon$ ensures $f$ is $\epsilon$-distribution calibrated with respect to $\Gamma$ by construction.
    Therefore, we have $\|\Gamma(p) - \Gamma(q)\| > C \epsilon$, and the result follows.
\end{proof}

While \Cref{prop:Lipschitz-lb} is a negative result in the sense that one can have high $\Gamma$-miscalibration with low distribution calibration, this result as emphasizes how distribution calibration might be misleading, and cannot preserve decision-theoretic guarantees for highly variable properties $\Gamma$.
See \Cref{fig:lb-intuition} in \Cref{app:lb-intuition} for visual intuition about this bound.

\subsection{$\Gamma$-calibration for discrete decisions $\gamma$}\label{subsec:Gamma-gamma}
We now turn our focus from the $\Gamma$-calibration of \emph{distributional} predictors $f : \X \to \simplex$ to the $\gamma$-calibration of $\Gamma$-predictors $g : \X \to \reals$.
For intuition, one can think of a predictor $g : \X \to \reals$ either as an intermediate post-processing of a distributional predictor $g = \Gamma \circ f$, or as a directly-learned artifact of empirical when minimizing the surrogate loss $L$ eliciting $\Gamma$.

Observe that some increase in calibration error is inevitable with the introduction of discretization.
\citet{hu_predict_2024} introduce \emph{calibration decision loss} (CDL) to this effect, which defines calibration error as the worst case error induced by discretization over all possible thresholdings in the binary setting.

\begin{lemmaE}[][end, restate]\label{lemma:Gamma-cal-bound-prob-misclass}
    Suppose a discrete property $\gamma: \simplex \to \R$ is refined by Lipschitz continuous property $\Gamma : \simplex \to \reals$ through the link $\psi : \reals \to \R$.
    Let $g : \X \to \reals$ be a scalar predictor.
    For each $u \in \im(g)$, let $\partial(u)$ denote the closest level set boundary $\Gamma^{-1}(u)$ such that $\Gamma^{-1}(u) = \gamma^{-1}(r_i) \cap \gamma^{-1}(r_{i+1})$, and $\delta(u) = \|u - \partial(u)\|$.
    Then for any $p \in \simplex$, we have $\Pr_{X,Y} [\gamma(p) \neq \psi(g(X))] \leq \Pr[\delta(g(X)) < t] + \Pr[\|\Gamma(p) - g(X)\| \geq t]$ for all $t > 0$.
\end{lemmaE}
\begin{proofE}
Observe that for any $u \in \reals$ and $p \in \simplex$, we have $\|\Gamma(p) - u\| \leq \delta(u) \implies \gamma(p) = \psi(u)$.

Therefore
    \begin{align*}
       \Pr_{X,Y}[\gamma(p) \neq \psi(g(X))] &\leq \Pr_{X,Y}[\|\Gamma(p) - g(X)\| \geq \delta(g(X))]\\
       &\leq \Pr[\delta(g(X)) < t] + \Pr[\|\Gamma(p) - g(X)\| \geq t] \qquad \forall t > 0 
    \end{align*}
    The second inequality follows from decomposing the first statement into two events.    
\end{proofE}

Before proceeding to bound discrete calibration error, we note that without distributional assumptions, our bound might be unavoidably vacuous as a result of discretization.
In \Cref{thm:approx-discrete-calibration}, we assume $D_{Y| \{\hat x : g(\hat x) = u\}}$ is Lipschitz in $u$; without such an assumption, the additive term in our bound is $2K$.
Similar assumptions are common in the literature~\citep{rossellini2025can,hartline2025smooth,blasiok2024smooth}.
As $K$ is typically greater than $\frac 1 2$, the second term in the bound would be greater than $1$.
To gain some intuition for why this bound can be vacuous, consider the property $\gamma(p) = \mathbbm{1}(\E_{Y \sim p} Y \geq \frac 3 2)$ for outcomes in $\Y = \{1,2,3\}$ and smooth property $\Gamma(p) = \E_{Y \sim p} Y$.
If there is one feature $\X = \{x\}$ and a scalar predictor $g(x) = \frac 3 2 - \frac {\epsilon}{2}$, while we have $\Gamma(D_{Y | \{x : g(x) = \frac 3 2 - \frac {\epsilon}{2}\}}) = \frac 3 2 + \frac {\epsilon}{2}$, then we have $\|\Gamma(D_{Y | \{x : g(x) = \frac 3 2 - \frac {\epsilon}{2}\}}) - g(x)\| = \epsilon$. 
This counterexample is an artifact of the $\delta_{\min}$ denominator (smallest distance from a prediction to discrete boundary) being small --- here, $\frac \epsilon 2$--- a value of $\delta_{\min} \leq \epsilon$ makes the bound vacuous, and further guarantees on calibration are impossible.

\begin{theorem}[Calibration bounds induced by discretization]\label{thm:approx-discrete-calibration}
    Consider a $K$-Lipschitz continuous property $\Gamma : \simplex \to \reals$ refining $\gamma : \simplex \to \R$ by the link $\psi: \reals\to\R$, and predictor $g : \X \to \reals$ that is $\epsilon$-approximately $\Gamma$ calibrated.
    Moreover, assume $D_{Y | \{\hat x : g(\hat x) = u\}}$ is $C$-Lipschitz continuous in $u$, and $\mathbf{diam}(\gamma, \psi) := \max_{r \in \im(\gamma)}\max_{u,u' : \psi(u) = r}|u - u'|$.
    Then for all $t > 0$, the post-processed predictor $\psi \circ g$ yields $\Pr_{X,Y}[\psi(g(X)) = \gamma(D_{Y | \{\hat x : \psi (g(\hat x)) =\psi(g(X))\}})]\geq 1 - (\Pr[\delta(g(X)) < t] + \frac{\epsilon + KC\mathbf{diam}(\gamma, \psi)}{t})$.
    In particular, if $\mathbf{im}(g)$ is finite, and $\delta_{\min} := \min_{u \in \im(g)} \delta(u)$, then $\Pr_{X,Y}[\psi(g(X)) = \gamma(D_{Y : \{\hat x : \psi (g(\hat x)) =\psi(g(X))\}})]\geq 1 - \frac{\epsilon + KC\mathbf{diam}(\gamma, \psi)}{\delta_{\min}}$.
\end{theorem}

\begin{proof}

\begin{align*}
    &\E_{X,Y}\|g(X) - \Gamma(D_{Y|\{\hat x : \psi(g(\hat x)) = \psi(g(X))\}})\| \\&\leq \E_{X,Y}\|g(X) - \Gamma(D_{Y|\{\hat x : g(\hat x) = g(X)\}})\| + \E_{X,Y}\|\Gamma(D_{Y|\{\hat x : \psi(g(\hat x)) = \psi(g(X))\}}) - \Gamma(D_{Y|\{\hat x : g(\hat x) = g(X)\}})\| \\
    &\leq \epsilon + \E_{X,Y}\|\Gamma(D_{Y|\{\hat x : \psi(g(\hat x)) = \psi(g(X))\}}) - \Gamma(D_{Y|\{\hat x : g(\hat x) = g(X)\}})\| \\
    &\leq \epsilon + K \E_{X,Y}\|D_{Y|\{\hat x : \psi(g(\hat x)) = \psi(g(X))\}} - D_{Y|\{\hat x : g(\hat x) = g(X)\}}\| \\
    &\leq \epsilon + K C\mathbf{diam}(\gamma,\psi)
\end{align*}
Note the last inequality follows from $C$-Lipschitz continuity of the marginal distribution in prediction space, and the diameter is the maximal distance between any predictions mapping to the same report $r \in \im(\gamma)$.
With this bound on distance from $g(X)$ to $\Gamma(D_{Y | \{\hat x : \psi(g(\hat x)) = \psi(g(X))\}})$, we can apply \Cref{lemma:Gamma-cal-bound-prob-misclass} to observe

\begin{align*}
       &\Pr_{X,Y}[\gamma(D_{Y|\{\hat x : \psi(g(\hat x)) = \psi(g(X))\}}) \neq \psi(g(X))]
       \\
       &\leq \Pr[\|\Gamma(D_{Y|\{\hat x : \psi(g(\hat x)) = \psi(g(X))\}}) - g(X)\| \geq \delta(g(X))]\\
       &\leq \Pr[\delta(g(X)) < t] + \Pr[\|\Gamma(D_{Y|\{\hat x : \psi(g(\hat x)) = \psi(g(X))\}}) - g(X)\| \geq t] \qquad \forall t > 0 \\
       &\leq \Pr[\delta(g(X)) < t] + \frac{\epsilon + K C\mathbf{diam}(\gamma,\psi)}{t}~.
    \end{align*}  
    The last part of the theorem statement holds by setting $t = \delta_{\min}$ and observing  $\Pr[\delta(g(X)) < \delta_{\min}] = 0$.
\end{proof}

Importantly, as $\mathbf{diam}(\gamma, \psi)$ shrinks, one could intuitively think of $\gamma$ as ``approaching'' $\Gamma$. 
In this case, the second term in the numerator disappears. 

These results can be interpreted in two primary ways: first, the bounds from \Cref{subsec:dist-Gamma} and \Cref{subsec:Gamma-gamma} can be composed to conclude that a $\epsilon$-distribution calibrated predictor $f : \X \to \simplex$ can be post-processed to conclude that $\psi \circ \Gamma \circ f$ is $\frac{(K \epsilon + CK\mathbf{diam}(\gamma, \psi))}{\delta_{\min}}$-approximately $\gamma$ calibrated.
In contrast, we can apply \Cref{thm:approx-discrete-calibration} directly to a scalar predictor $g : \X \to \reals$.
As discussed after \Cref{prop:Lipschitz-lb}, the Lipschitz constant of the smoothed property $\Gamma$ might impose a \emph{cost} to calibration by predicting property values directly, so without making stringent distributional assumptions, it is hard to generally compare the $\Gamma$-miscalibration of two predictors $f : \X \to \simplex$ and $g : \X \to \simplex$.
That is, we cannot say anything about the ordinal relationship between calibration error for a distributional predictor attained from ERM on a $n$-dimensional score-based loss $\epsilon^f$ and calibration error for a scalar predictor attained by ERM on a $\Gamma$-specific loss $\epsilon^g$. 
Recall, however, that gradient-based optimization methods on $g : \X \to \reals^d$ are generally more efficient than the same methods applied to $f : \X \to \simplex$ for $d \ll n$.
Moreover, straightforward extensions of the results from \citet{collina2026samplecomplexitymulticalibration} imply a sample complexity of $\epsilon^{-(d+2)}$, exponentially improving in prediction dimension $d$.

\section{Discussion}
We have established bounds on the approximate calibration of discrete properties, particularly for \emph{strongly orderable} discrete properties.
This work is the first to our knowledge to establish such approximate calibration bounds for discrete properties using norm-based calibration error metrics, in contrast to the worst-case miscalibration error metrics in CDL~\citep{hu_predict_2024} and CCE~\citep{rossellini2025can}.
In order to establish this approximate bound, we ``smoothed'' discrete properties into Lipschitz continuous properties whose level sets can be mapped back to the discrete property, and established approximate calibration of the smoothed properties.

\paragraph{Future work}
Many directions of future work remain: first, our characterization of Lipschitz elicitation complexity is limited to one dimension, as is \Cref{thm:approx-discrete-calibration}, relying on strong orderability of a property $\gamma$, seen in the presence of $\mathbf{diam}(\gamma, \psi)$ in the bound in \Cref{thm:approx-discrete-calibration}. 
An interesting immediate direction of future work is to extend our characterization to general $d$-dimensional bounds.

Another line of future work emerges from the observation that one advantage of distribution calibration (with respect to $\gamma$) emerges from strong connections between (multi)calibration and to omniprediction, where a model can be post-processed to yield decision-theoretic guarantees for a wide suite of properties $\gamma$. 
We conjecture this characterization might explain some of the results of the efficiency of the \texttt{TreeCal} algorithm of \citet{fishelsonhigh} also used by \citet{peng2025high}.

Finally, one might consider modifications of our proposed algorithms which give the \emph{smoothest} possible properties, possibly by modifying interpolation points to not be midpoints in \Cref{alg:smooth-embedding} or uniformly spaced in \Cref{alg:piecewise-RoE-normals}, tightening the bounds derived in \Cref{sec:calibration}.

\paragraph{Broader impacts}
While our work is theoretical in nature, our results make salient the deception of ``low calibration error'' made possible by simply picking a property $\Gamma$ that is quite smooth.
We view this as not unique to our work, but do advise proceeding with caution when drawing conclusions based on low calibration error. 

\begin{ack}
We would like to thank Rabanus Derr, Sanket Shah, Rafael Frongillo, Bo Waggoner, Georgy Noarov, Natalie Collina, Sarah Fleming, and Patrick Lanza for helpful discussions and feedback.
\end{ack}

\newpage
\bibliographystyle{abbrvnat}
\bibliography{refs_manual}

\newpage
\appendix

\section{Omitted proofs}\label{app:omitted-proofs}
\printProofs

\section{Omitted discussion of calibration error metrics}\label{app:calibration-metrics}
\paragraph{Calibration error metrics}
It is worth discussing the emergence of \emph{calibration error metrics} as a unit of study in their own right.
\citet{rossellini2025can} pose axioms of calibration metrics of \emph{testability} --- the ability to estimate a metric with finitely many samples --- and \emph{actionability} --- obtaining decision-theoretic guarantees from an approximately ``calibrated'' predictor.
\citet{haghtalab2024truthfulness} discusses \emph{truthfulness} of calibration metrics, lated extended  by \citet{qiao25truthfulness}, who apply their results to decision-theoretic (\emph{actionable}) calibration metrics.
While expected calibration error \citep{naeini2015obtaining} is the ``canonical'' calibration error metric, it is not empirically testable with finite samples because of a binning requirement.
In contrast, \citet{blasiok_unifying_2023} propose Distance to Calibration as a testable calibration error metric, but it is notably not actionable.
Other decision-theoretic calibration metrics have emerged in binary settings, like Calibration Decision Loss~\citep{hu_predict_2024} and Cutoff Calibration Error~\citep{rossellini2025can}, but for the purposes of this paper, we generally consider calibration error measured by an expected norm.
We leave it as a further line of inquiry to understand formally which calibration error metrics are suitable for models which estimate properties $\Gamma(p)$ instead of distributions $p$.
Notably \citet{bailie2025propertyelicitationimpreciseprobabilities} recently example the elicitation of ``imprecise probabilities,'' which may be helpful for this problem.

\section{Applications of Algorithms to \Cref{eq:motivating-example}}\label{app:apply-algos}
\subsection{\Cref{alg:smooth-embedding} applied to \Cref{eq:motivating-example}}
To demonstrate how \Cref{alg:smooth-embedding} works, we revisit \Cref{eq:motivating-example}, where $
    L(u, y) := \begin{cases}
    \max(-3 u, u, 3u - 6)& y = 1\\
    \max(3-3u, \frac 1 2 u - \frac 1 2, 3u - 8)& y = 2 \\
    \max(-3u + 5, 5-2u, \frac 9 2 - \frac 3 2 u, 3u - 9)& y = 3
\end{cases}$ as constructed from \citet[Theorem 11]{finocchiaro_embedding_2020} embeds $\ell$, with nondifferentiable embedding points $\varphi(\R) = \{0, 1, 3\}$. \\
We construct the interpolation set with these embedding points, and the (ordered) midpoints to obtain $U = \{0,\frac{1}{2}, 1, 2, 3\}$. 
Moreover, we take the ``pseudo-identification'' function $V$ and interpolate to obtain $\bar V$ for outcome $y = 1$ 

\begin{align}
V(u,1) &= \begin{cases}
    -3 & u < 0 \\
    0 & u = 0 \\
    1 & u \in (0,3) \\
    2 & u = 3 \\
    3 & u > 3
\end{cases}
&
\bar V(u,1) &= \begin{cases}
    u & u < 0 \\
    2u & u \in [0, \frac 1 2]\\
    1 & u \in (\frac 1 2, 2] \\
    u - 1 & u > 2
\end{cases}~,
\end{align} and likewise for $y = 2, 3$.

Integrating $\bar V(u,y)$ yields a loss $\bar L(u,y) = \int_0^u \bar V(z, y) dz \implies$
\begin{align*}
    \bar L(u,1) &= \begin{cases}
    \frac{u^2}2 & u < 0 \\
    u^2 & u \in [0, \frac 1 2]\\
    u - \frac 1 4 & u \in (\frac 1 2, 2] \\
    \frac{u^2}2 - u + \frac 7 4 & u > 2
\end{cases} &
&\text{eliciting } &
\Gamma(p) &= \begin{cases}
    \frac{6p_2 + 5p_3}{(4p_1 + 2p_3)} & 4p_1 + 3p_2 \leq 1 \\
    \frac{-4p_1 + 24p_2 + 9p_3}{(24p_2+2p_3)} & 4p_2 + 3p_3 \geq 1 \wedge p_2 + \frac{11}{4}p_3 \leq 1 \\
    \frac{-4p_1+2p_2+8p_3}{(2p_2+p_3)} & p_2 + \frac{11}{4}p_3\geq 1 \wedge p_2 + 5p_3 \leq 2 \\
    \frac{4p_1+8p_2+18p_3}{(4p_1+5p_2+6p_3)} & p_2 + 5p_3 \geq 2
\end{cases}
\end{align*}

Importantly, since $\bar V(u,y)$ is continuous in its first argument, we can apply the scoring rule characterization of \citet[Theorem 5]{steinwart_elicitation_2014} to claim that $\bar V$ is an \emph{identification function}, and $\bar L(u,y) = \int_0^u \bar V(z, y) dz$ elicits the property which is the root of $\E_{Y \sim p} \bar V(\cdot, Y)$.

By design, the interpolation points $\frac 1 2$ and $2$ correspond to the reports where the set-valued inverse $\Gamma^{-1}(\frac 1 2) = \{p \in \simplex : -3 p_1 + p_2 = -2\}$ and $\Gamma^{-1}(2) = \{p \in \simplex : 5 p_1 + 4p_2 = 3\}$.

Therefore, we construct the link $\psi$ such that $\psi \circ \Gamma(p) \in \gamma(p)$ (with equality up to the boundaries of set-valued inverses) will have $\psi(r) = \begin{cases}
    1 & r \leq \frac 1 2 \\ 
    2 & r \in (\frac 1 2, 2]\\
    3 & r >  2
\end{cases}$.


\subsection{\Cref{alg:piecewise-RoE-normals} applied to \Cref{eq:motivating-example}}

To demonstrate how \Cref{alg:piecewise-RoE-normals} works, we revisit \Cref{eq:motivating-example}.
With the property 
\begin{align*}\gamma(p) = \begin{cases}
    1 & p_2 \leq -2 + 3 p_1 \\
    2 & -2 + 3 p_1 < p_2 \leq \frac{3-5p_1}{4}\\
    3 & p_2 \geq \frac{3-5p_1}{4}\\
\end{cases}~,~
\end{align*} 
we can take the distributions $p_{11} = [0.7, 0.1, 0.2]$ and $p_{12} = [0.68, 0.04, 0.28]$ and use SVD to obtain
$o_1 = \frac 1 {\sqrt{14}}[-1,3,2]$ and with $p_{21} = [0.5, 0.125, 0.375]$ and $p_{22} = [.25, .4375, .3125]$ to obtain $o_2 = \frac{1}{\sqrt{14}}[-2,-1, 3]$.

We can verify these are both oriented, and derive the identification function
\begin{align}
    V(r, y) &=  \begin{cases}
    r - o_{1,y} & r \leq 0 \\
    -o_{1,y} + (o_{1,y} - o_{2,y}) r & r \in [0,1]\\
    r - 1 - o_{2,y}& r \geq 1
    \end{cases}
    &
\Gamma(p) &= \begin{cases}
\inprod{o_1}{p} & p_2 \leq -2 + 3 p_1 \\ 
\frac{\inprod{o_1}{p}}{\inprod{p}{o_1 - o_2}} & -2 + 3 p_1 \leq p_2 \leq \frac{3-5p_1}{4}\\
\inprod{o_2}{p}+ 1 & p_2 \geq \frac{3-5p_1}{4}\\
\end{cases}\label{eq:Gamma-norm}
\end{align}

First, we can verify that when $p_2 = 3p_1 - 2$, by simplex constraints, we have $p_3 = 3 - 4p_1$, and we can parameterize $\gamma^{-1}(1) \cap \gamma^{-1}(2) = \{(p_1, 3p_1-2, 3-4p_1) : p_1 \in [\frac 2 3, \frac 3 4]\}$ with one variable.
With this parameterization, we have $\inprod{o_1}{p} = 0 = \frac{\inprod{o_1}{p}}{\inprod{p}{o_1 - o_2}}$ for all $p \in \gamma^{-1}(1) \cap \gamma^{-1}(2)$, and similarly for the boundary $p_2 = \frac{3-5p_1}{4}$.

\section{Visual intuition for $\Gamma$-calibration lower bound}\label{app:lb-intuition}
\begin{figure}[h]
    \centering
    \resizebox{0.8\linewidth}{!}{\definecolor{purple}{RGB}{209, 120, 250}

\begin{tikzpicture}[scale=8]
    \coordinate (A) at (0, 0); 
    \coordinate (B) at (1, 0); 
    \coordinate (C) at (0.5, {sqrt(3)/2});

    \coordinate (B12_AC) at (0.125, {0.25*sqrt(3)/2}); 
    \coordinate (B12_AB) at (0.333, 0);               
    
    \coordinate (B23_AB) at ({4/11}, 0);              
    
    \coordinate (B34_BC) at (0.625, {0.75*sqrt(3)/2}); 
    \coordinate (B34_AB) at (0.4, 0);                  

    \fill[blue!10] (A) -- (B12_AC) -- (B12_AB) -- cycle;
    \fill[green!10] (B12_AC) -- (C) -- (B34_BC) -- (B34_AB) -- (B12_AB) -- cycle;
    \fill[orange!10] (B34_BC) -- (B) -- (B34_AB) -- cycle;

    \foreach \c in {-.95, -.9,..., 0} {
        \pgfmathsetmacro{\pOne}{(\c*3.74165 + 1)/4}
        \pgfmathsetmacro{\pTwo}{(\c*3.74165 + 1)/3}
        \begin{scope}
            \clip (A) -- (B12_AC) -- (B12_AB) -- cycle;
            \draw[blue!70!black, thin] ({0.5*\pOne}, {\pOne*sqrt(3)/2}) -- (\pTwo, 0);
        \end{scope}
    }

    \foreach \c in {.05, .1, ..., 1} {
        \pgfmathsetmacro{\pOne}{(\c + 1)/(4 - 3*\c)}
        \pgfmathsetmacro{\pTwo}{(\c + 1)/(3 + 2*\c)}
        \begin{scope}
            \clip (B12_AC) -- (C) -- (B34_BC) -- (B34_AB) -- (B12_AB) -- cycle;
            \draw[green!60!black, thin] ({0.5*\pOne}, {\pOne*sqrt(3)/2}) -- (\pTwo, 0);
        \end{scope}
    }

    \foreach \c in {.6} {
        \pgfmathsetmacro{\pOne}{(\c + 1)/(4 - 3*\c)}
        \pgfmathsetmacro{\pTwo}{(\c + 1)/(3 + 2*\c)}
        \begin{scope}
            \clip (B12_AC) -- (C) -- (B34_BC) -- (B34_AB) -- (B12_AB) -- cycle;
            \draw[green!40!black, thick] ({0.5*\pOne}, {\pOne*sqrt(3)/2}) -- (\pTwo, 0) node[left, pos=0.5] {\scriptsize $\epsilon_2 = 0$};
        \end{scope}
    }

    \foreach \c in {1.05, 1.1, ..., 2.0} {
        \pgfmathsetmacro{\pOne}{((\c-1)*3.74165 + 2)}
        \pgfmathsetmacro{\pTwo}{((\c-1)*3.74165 + 2)/5}
        \begin{scope}
            \clip (B34_BC) -- (B) -- (B34_AB) -- cycle;
            \draw[orange!70!black, thin] ({0.5*\pOne}, {\pOne*sqrt(3)/2}) -- (\pTwo, 0);
        \end{scope}
    }

    \draw[black, thick] (B12_AC) -- (B12_AB);
    \draw[black, thick] (B34_BC) -- (B34_AB);

    \draw[color=black, thick] (.38,0.02) -- (.355, .0575) node[above] {\scriptsize $\epsilon_1 = 0.04$};
    \draw[color=black, thick] (.38,0.02) -- (.42, .02) node[below, yshift=-3pt] {\scriptsize $\epsilon_3 = 0.43$};
    
    \node[purple] at (.38, .02){{\Large$\bullet$}};
    \node[purple] at (.42, .02){{\LARGE$\star$}};
    \node[purple] at (.367, .5){{\Large$\spadesuit$}};
    \draw[color=purple, dashed] (.38,0.02) circle [radius=.04];

    \draw[very thick] (A) -- (B) -- (C) -- cycle;
    \node[below left] at (A) {$Y = 1$};
    \node[below right] at (B) {$Y = 3$};
    \node[above] at (C) {$Y = 2$};
\end{tikzpicture}}
    \caption{Let $\bullet$ denote a prediction $f(x) \in \simplex$, and consider two alternate marginal distributions $D_{Y | \{\hat x : \Gamma(f(\hat x)) = \Gamma(\bullet)\}}$, $\star$ and $\spadesuit$.}
    \label{fig:lb-intuition}
\end{figure}
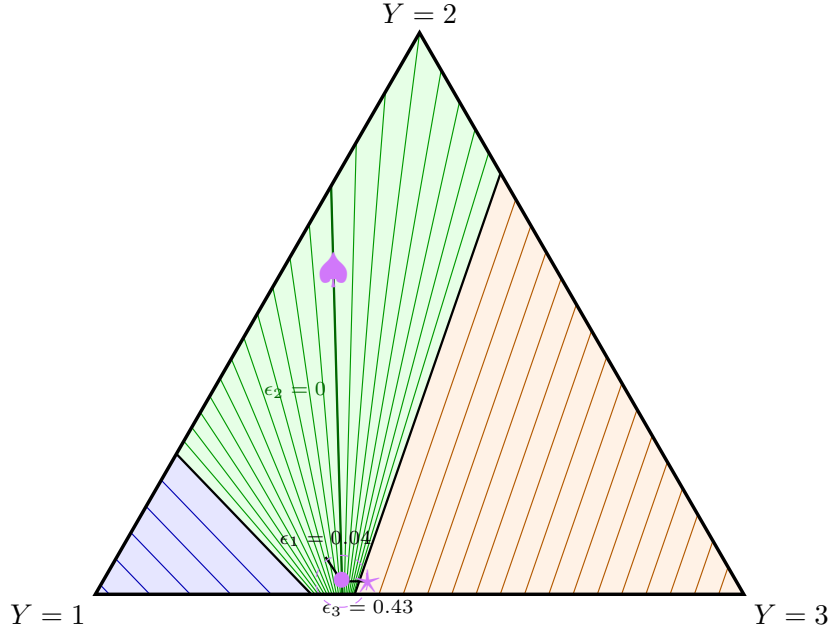
Consider a predictor $f(x) = \bullet$ as in \Cref{fig:lb-intuition}, where $\X = \{x\}$ for simplicity.
Moreover, we compare two alternate marginal distributions $D_{Y | \{\hat x : \Gamma(f(\hat x)) = \Gamma(\bullet)}\}$, $\star$ and $\spadesuit$, where $\Gamma$ is the smooth property returned from \Cref{alg:piecewise-RoE-normals} in \Cref{app:apply-algos}.

If $D_{Y | \{\hat x : \Gamma(f(\hat x)) = \Gamma(\bullet)\}} = \star$, then $f$ is $\epsilon_1 = 0.04$ distribution calibrated with respect to $\Gamma$, but it is $\epsilon_3 \approx 0.43$-approximately $\Gamma$-calibrated.
This jump in the calibration error emerges as $\bullet$ sits on the ``most non-smooth'' region of $\Gamma$.
In contrast, if $D_{Y | \{\hat x : \Gamma(f(\hat x)) = \Gamma(\bullet)\}} = \spadesuit$, then $\Gamma \circ f$ is approximately $0.48$-distribution calibrated with respect to $\Gamma$, yet perfectly $\Gamma$-calibrated.



\end{document}